\RequirePackage{fix-cm}
\documentclass[smallcondensed]{svjour3}     
\smartqed  
\usepackage{graphicx}
\usepackage{lipsum}
\usepackage{amsfonts}
\usepackage{graphicx}
\usepackage{epstopdf}
\usepackage[symbol]{footmisc}
\usepackage{algorithm,algorithmic}
\usepackage{listings}
\usepackage{booktabs}
\usepackage{subfig}
\usepackage{epsfig}
\usepackage{color}
\usepackage{float}
\usepackage{hyperref}
\usepackage{footnote}
\usepackage{mathrsfs}
\usepackage{booktabs}
\usepackage{listings}
\usepackage{amsmath}
\usepackage{amssymb}
\usepackage{placeins}
\newcommand{\B}{\mathcal{B}}
\newcommand{\relu}{\text{ReLU}}
\newcommand{\SP}{\text{SP}}
\newtheorem{assumption}[theorem]{Assumption}
\begin{document}

\title{Generalization Error Analysis of Neural networks with Gradient Based Regularization
}


\author{Lingfeng Li   \and
        Xue-Cheng Tai \and Jiang Yang
}


\institute{Lingfeng Li \at
              Department of Mathematics, Hong Kong Baptist University, Hong Kong, China\\
              Department of Mathematics, Southern University of Science and Technology, Shenzhen, China
              \email{lingfengli@life.hkbu.edu.hk}           
           \and
           Xue-Cheng Tai \at
           Department of Mathematics, Hong Kong Baptist University, Hong Kong, China\\
           \email{xuechengtai@hkbu.edu.hk}\\
           \and
           Jiang Yang \at
           Department of Mathematics, Southern University of Science and Technology, Shenzhen, China\\
           \email{yangj7@sustech.edu.cn}
}

\date{Received: date / Accepted: date}

\maketitle

\begin{abstract}
We study gradient-based regularization methods for neural networks. We mainly focus on two regularization methods: the total variation and the Tikhonov regularization. Applying these methods is equivalent to using neural networks to solve some partial differential equations, mostly in high dimensions in practical applications. In this work, we introduce a general framework to analyze the generalization error of regularized networks. The error estimate relies on two assumptions on the approximation error and the quadrature error. Moreover, we conduct some experiments on the image classification tasks to show that gradient-based methods can significantly improve the generalization ability and adversarial robustness of neural networks. A graphical extension of the gradient-based methods are also considered in the experiments.
\end{abstract}

\section{Introduction}
Deep neural networks (DNNs), which are compositions of some linear and non-linear mappings, have become the most popular tool in artificial intelligence. Its ability to fit complex functions helps it achieve state-of-the-art performances and beat other methods by a huge margin in many areas, such as image processing, video processing, and natural language processing. For a complete introduction to deep learning, one can refer to \cite{goodfellow2016deep,higham2019deep}. Recently, some applied mathematicians have also successfully applied DNNs to solve some partial differential equations (PDEs) \cite{sirignano2018dgm,weinan2018deep,han2018solving}. The main advantage of DNNs is that they can solve very high-dimensional problems that are intractable for traditional numerical methods. Some literature about the expressive power of DNNs are \cite{barron1993universal,cybenko1989approximation,hornik1989multilayer,zhou2020universality}.

Though DNNs have such great expressive power, the training of DNNs is not easy, especially for those having enormous amounts of parameters. One of the most common issues in training DNNs is overfitting, i.e., the model fits the training data well but performs poorly on the testing data. Overfitting is very likely to happen when the number of parameters is significantly larger than the number of training samples. To prevent models from overfitting, some regularization techniques are usually applied. Two types of regularization methods are commonly used in practice: implicit regularization and explicit regularization. The implicit method is often induced by the optimization scheme, e.g., the early stopping method. The explicit method is introducing a regularization term directly into the loss function.

Another issue we concern about the DNNs is stability, which is also referred to as adversarial robustness. It has been shown that DNNs can be fooled by adding very small perturbation to the inputs \cite{goodfellow2015explaining,szegedy2013intriguing}. The algorithms to find such a perturbation is called adversarial attacks, such as FGSM \cite{goodfellow2015explaining}, PGD \cite{madry2018towards}, and one-pixel attack \cite{su2019one}. To improve the robustness of our networks, various adversarial defensive strategies have been developed \cite{goodfellow2015explaining,kannan2018adversarial}, and most of them are in the form of explicit regularization.

Regularization plays important roles not only in deep learning but also in variational models where the regularization terms are often designed based on some prior knowledge \cite{rudin1992nonlinear,chan2001active,vese2002multiphase}. One of the most popular regularization methods is the gradient-based regularization like the total variation (TV) \cite{rudin1992nonlinear} and the Tikhonov regularization \cite{tikhonov2013numerical}. The Euler-Lagrangian equations of the regularized problems are some PDEs. Numerical methods like finite difference are commonly used to solve them. However, due to the curse of dimensionality, most of the traditional numerical methods are only able to handle low-dimensional problems. Therefore, using DNNs to approximate solutions to variational models and PDEs have attracted extensive attentions in recent years \cite{weinan2018deep,han2018solving,sirignano2018dgm,oberman2020partial,raissi2018hidden,raissi2019physics}.

For classical numerical methods, error analysis is one of the most critical parts of the research. Usually, we would expect an explicit bound on the difference between the numerical solution and the true solution. If the error bound depends on the computed numerical solution, we call it a posterior error. Otherwise, we call it a priori error. However, because of the high complexities of neural networks and the randomness induced by the optimization scheme, obtaining such type of error bound for deep learning algorithms is not feasible in general. Instead, we can analyze the generalization error. Some generalization error analysis have been developed for some PDE-related problems \cite{mishra2020estimates,mishra2020estimates2,lu2021priori,muller2021error}, but the study of the generalization error for general variational problems is still very few, especially for total variation related problems.

In this work, We introduce a general framework for the generalization error analysis when DNNs are used to solve variational models. Our error estimate consists of two parts: the approximation error and the quadrature error. We also derive an error bound for two-layer neural networks using the proposed framework. In the numerical experiments, we test the DNNs trained with the TV and the Tikhonov regularization for image classification problem. We also extend the TV and Tikhonov regularization methods to problems defined on graphs. Our results suggest that the gradient-based methods can significantly improve both the generalization ability and stability of DNNs, especially when the training samples are few. 

The rest of this paper is organized as follows. In section \ref{sec2}, we first give a brief introduction to some popular explicit regularization methods and adversarial training methods in deep learning. Then we present the definition of the total variation and Tikhonov regularization along with their extensions for problems defined on graphs. In section \ref{sec3}, we develop our generalization error analysis for some regularized variational models. In section \ref{sec4}, we apply the proposed framework to derive an error estimate for the two-layer neural networks. Some numerical experiments of image classification tasks are conducted in section \ref{sec5} to show the effectiveness of the gradient-based regularization methods for DNNs. Lastly, some discussions and conclusions are given in section \ref{sec6}.

\section{Preliminaries} \label{sec2}
\subsection{Explicit regularization and adversarial training for DNN}
Let us start with a $n$-phase classification problem in deep learning. We assume the data belong to a bounded domain $\Omega\subset\mathbb{R}^d$ and follow a probability distribution $\rho_x(\Omega)$. There exist a ground true solution $y(x)$ that assign each $x\in\Omega$ an integer label. To solve the classification problems, we define a neural network $\phi(x;\theta):\Omega\rightarrow\mathbb{R}^n$ where $\theta$ denotes the set of parameters in $\phi$. Let $\phi_k(x;\theta)$ be the $k$th component of $\phi(x;\theta)$ which represents the likelihood of $x$ belonging to the $k$th class. Then, the classification result of $\phi$ is obtained by $\arg\max_k\{\phi_k(x;\theta)\}$. When $n=2$, we can use a scalar-valued network to solve the problem and the labeling result is obtained by taking the sign of its output. 

When training the neural network, we need a loss function $D(\phi(x;\theta), y(x))$ to measure how well the network outputs $\phi$ fit the ground truth $y$. Some popular choices for $D$ are $l_1$ loss, $l_2$ loss, and the cross-entropy loss. Our  goal is to find an optimal $\theta$ to minimize the expected loss:
\begin{equation}
    L(\theta)=\mathbb{E}_{\rho_x}\left[D(\phi(x;\theta),y(x))\right]. \label{eq:gen_loss}
\end{equation}
In practice, we only observe a set of training data sampled from $\rho_x$: $\{(x_i,y_i)\}_{i=1}^N$ where $y_i=y(x_i)$. Therefore, the loss function we are actually working with is the empirical loss:
\begin{equation}
    L_{N}(\theta)=\frac{1}{N}\sum_{i=1}^N D(\phi(x_i;\theta),y_i), \label{eq:train_loss}
\end{equation}
which is a discretized version of the expected loss (\ref{eq:gen_loss}).

A very common issue in machine learning is the overfitting. When the number of training samples is significantly less than the number of parameters, we may easily achieve very small empirical loss $L_{N}(\theta)$ but the expected loss $L(\theta)$ remains very large. To prevent models from overfitting the data, some regularization techniques are necessary during the training procedure. Generally, there are two types of regularization methods, namely the implicit regularization and the explicit regularization. The implicit method is usually induced by optimization algorithms, e.g., early stopping strategies. The explicit method is directly adding a regularization term $R(\phi(x;\theta))$ to the loss function:
\begin{equation}
    L_N^R(\theta)=\frac{1}{N}\sum_{i=1}^N D(\phi(x_i;\theta),y_i) + \alpha R(\phi(x;\theta)), \label{eq:explicit_regularization}
\end{equation}
where $\alpha>0$. In this work, we mainly focus on the explicit method. The simplest explicit regularization method for deep learning is the weight decay:
\begin{equation}
    \frac{1}{N}\sum_{i=1}^N D(\phi(x_i;\theta),y_i) + \alpha\Vert\theta\Vert^2. \label{eq:weight_dacay}
\end{equation}
Some other regularization methods have been proposed recently for defending adversarial attacks. Adversarial attacks are algorithms that find bounded perturbations of the input data to fool the neural networks. If a neural network is trained without any regularization or defensive strategies, its prediction results may easily be altered even under very small perturbations. A simple adversarial training method was first introduced to defend adversarial attacks. The idea is using the data term in (\ref{eq:explicit_regularization}) of perturbed images as the regularization:
\begin{equation}
    \frac{1}{N}\sum_{i=1}^N (1-\alpha)D(\phi(x_i;\theta),y_i) + \alpha D(\phi(x_i+\delta(x_i);\theta),y_i), \label{eq:adversarial_training}
\end{equation}
where $\delta(x_i)$ is the perturbation generated by some attack methods for the input sample $x_i$. By Taylor expansion, (\ref{eq:adversarial_training}) approximately equals to
\begin{equation*}
    \frac{1}{N}\sum_{i=1}^N D(\phi(x_i;\theta),y_i) + \alpha |\nabla D(\phi(x_i;\theta),y_i)|
\end{equation*}
or
\begin{equation*}
    \frac{1}{N}\sum_{i=1}^N D(\phi(x_i;\theta),y_i) + \alpha |\nabla D(\phi(x_i;\theta),y_i)|^2,
\end{equation*}
when $\delta=\nabla_{x_i}D/|\nabla_{x_i}D|$ or $\delta=\nabla_{x_i}D$ \cite{oberman2020partial}.
This method is straightforward and effective for defending adversarial attacks, so it is adopted as the baseline method in many research papers. Later, another popular regularization method, called adversarial logit pairing (ALP), was introduced in \cite{kannan2018adversarial}:
\begin{equation}
    \frac{1}{N}\sum_{i=1}^N \left[D(\phi(x_i;\theta),y_i) + \alpha| \phi(x_i+\delta(x_i);\theta)-\phi(x_i;\theta)|\right] \label{eq:alp},
\end{equation}
where $|\cdot|$ represents the vector $l_2$ norm and can be replaced by other norms.

\subsection{TV and Tikhonov regularization}
Variational models are used in many applications of image processing, such as denoising, inpainting and segmentation. Usually, we search for a solution $u(x):\Omega\rightarrow\mathbb{R}^n$ from a certain space $U$ by minimizing an energy functional. Similar to (\ref{eq:explicit_regularization}), the energy usually consists of a data fidelity term $D(u(x))$ and a regularization term $R(u(x))$:
\begin{equation*}
    \int_{\Omega} D(u(x)) + \alpha R(u(x)) dx, \quad  u\in U.
\end{equation*}
One of the earliest regularization method is the Tikhonov regularization:
\begin{equation*}
    R(u(x)) = \sum_{k=1}^n|\nabla u_k(x)|^2
\end{equation*}
for $n\geq 1$.
However, in many image processing tasks, using the Tikhonov regularization may oversmooth the solutions and blur the edges. The more appropriate regularization for image problems is the total variation which was first introduced by Rudin, Osher and Fatemi in the famous ROF model \cite{rudin1992nonlinear}. The TV regularization for classification problems is written as:
\begin{equation*}
    R(u(x)) = \sum_{k=1}^n|\nabla u_k(x)|
\end{equation*}
for $n\geq 1$. The gradient here is the distributional derivative, so the total variation can be defined for some discontinuous functions. When $u_k(x)$ is the indicator function of a region, the $\int_{\Omega}|\nabla u_k(x)|dx$ equals the length of the boundary of this region. Various efficient algorithms have been developed to solve the TV regularized models \cite{chambolle2004algorithm,wu2010augmented,goldstein2009split,chan1999nonlinear,yuan2010study}. In the early stage of the research of neural networks, the Tikhonov regularization method has also been introduced as the double backpropagation method \cite{drucker1992improving}:
\begin{equation*}
    L_N^{Tik}(\phi)=\frac{1}{N}\sum_{i=1}^N \left(D(\phi(x_i;\theta),y_i) + \alpha \sum_{k=1}^n |\nabla_{x}\phi_k(x_i;\theta)|^2\right), \quad \phi\in V,
\end{equation*}
where $V$ is a set of neural networks.
This method can reduce the generalization error in many numerical experiments. When the data term $D$ is in the quadratic form, it has also been shown that training with Tikhonov regularization is equivalent to adding random noise to the training data \cite{bishop1995training}. Similarly, we can also apply the TV regularization to neural networks training:
\begin{equation*}
    L^{TV}_N(\phi)=\frac{1}{N}\sum_{i=1}^N \left(D(\phi(x_i;\theta),y_i) + \alpha \sum_{k=1}^n |\nabla_{x}\phi_k(x_i;\theta)|\right), \quad \phi\in V.
\end{equation*}
\subsection{Graphical TV and Tikhonov regularization}
Graph models have many successful applications in data classification problems, especially in high dimensional cases. An undirected graph $G=(W,E,\omega)$ consists of a set of vertices $W=\{x_i\}_{i=1}^N$, a set of edges $E=\{(x_i,x_j)\}_{i\neq j}$, and the weights assigned to edges $\{\omega_{ij}\}_{i\neq j}$. In our setting, each vertex $x_i$ corresponds to a training sample. The weight $\omega_{ij}$ is usually a monotonic decreasing function of $d_{ij}$. Like the finite difference method for numerical methods of PDEs, we can define the differential operators on graphs in the discrete sense. Let $u:W\rightarrow\mathbb{R}$ be a function defined on $G$. The gradient of $u$ is defined as
\begin{equation*}
    \nabla u(x_i)(x_j) = \omega_{ij}(u(x_j)-u(x_i)).
\end{equation*}
Then, the total variation and Tikhonov regularization of $u$ are given as:
\begin{equation*}
    |\nabla u(x_i)| = \sum_{(i,j)\in E} \omega_{ij}|u(x_j)-u(x_i)|
\end{equation*}
and
\begin{equation*}
    |\nabla u(x_i)|^2 = \sum_{(i,j)\in E} \omega_{ij}|u(x_j)-u(x_i)|^2.
\end{equation*}
Notice that the total variation defined here is an anisotropic version instead of an isotropic version. Computing the full total variation or Tikhonov regularization may be expansive when $N$ is large, so some edges with small weights can be neglected. We usually select a set of nearest data to each $x_i$ as the neighbor set $N_i$, $i=1,2,\dots,N$. Graph models have been successfully applied in neural network training before to enhance the robustness \cite{wang2020adversarial}. One main challenge is how to determine the weight $\omega$ because it is not easy to define a suitable distance metric for high dimensional data. The most commonly used metrics is the Euclidean distance, but it is not appropriate for high dimensional data like images.

In this work, we adopt a very simple way to determine the weight without computing the distance. During the training stage, since we have access to the ground true labels of all training samples, we can simply assume that the distance between two data $d_{ij}$ is 0 if $y_i=y_j$, and is $\infty$ if $y_i\neq y_j$. Let $\omega_{ij}=\exp(-d_{ij})$, the weight can then be defined as
\begin{equation*}
    \omega_{ij}=\begin{cases} 1 &,y_i=y_j\\
     0 &,y_i\neq y_j
    \end{cases},
\end{equation*}
Consequently, the neighbor set $\mathbf{N}_i$ can simply be a set of samples belonging to the same class with $x_i$. The graphical TV and Tikhonov regularization for neural networks training are then written as:
\begin{equation*}
    L^{GTV}_N(\phi)=\frac{1}{N}\sum_{i=1}^N \left(D(\phi(x_i;\theta),y_i) + \alpha  \sum_{x_j\in\mathbf{N}_i}|\phi(x_j;\theta)-\phi(x_i;\theta)|\right), \quad \phi\in V
\end{equation*}
and
\begin{equation*}
    L^{GTik}_N(\phi)=\frac{1}{N}\sum_{i=1}^N \left(D(\phi(x_i;\theta),y_i) + \alpha  \sum_{x_j\in\mathbf{N}_i}|\phi(x_j;\theta)-\phi(x_i;\theta)|^2\right), \quad \phi\in V.
\end{equation*}

\section{Generalization error analysis for regularized neural networks} \label{sec3}
We consider the regularized expected loss:
\begin{equation}
    L^{R}(u)=\mathbb{E}_{\rho_x}\left(D(u(x),y(x))+\alpha R(u(x))\right), \quad u\in U\label{eq:energy_general}
\end{equation}
where $U$ id a function space and $R$ is the TV or Tikhonov regularization. For the sake of simplicity, we assume $\alpha=1$ in the rest of this paper and it is easy to extend our theory to any $\alpha>0$. Let $V\subset U$ be a set of neural networks with finite number of parameters. When training a neural network, what we actually minimize is the empirical loss which depends on a set of training data $\{x_i\}_{i=1}^N$:
\begin{equation}
    L^{R}_N(\phi)=\frac{1}{N}\sum_{i=1}^N\left(D(\phi(x_i;\theta),y(x_i))+ R(\phi(x_i;\theta))\right), \quad \phi\in V. \label{eq:energy_general_discrete}
\end{equation}
\begin{definition}[$\delta$-minimizer]
A function $\hat{u}\in U$ is said to be a $\delta$-minimizer of $L(u)$ if
\begin{equation*}
    L(\hat{u})\leq \inf_{u\in U}L(u)+\delta.
\end{equation*}
\end{definition}

Denote $\hat{u}=\arg\min_{u\in U}L^R(u)$ and $\hat{\phi}^N\in V$ be the $\delta$-minimizer of $L^R_N(\phi)$. For a general variational problem, we are interested in the generalization error, i.e, $\mathcal{E}(\hat{\phi}^N,\hat{u})=L^R(\hat{\phi}^N)-L^R(\hat{u})$. If $\mathcal{E}(\hat{\phi}^N,\hat{u})$ is small, we expect the network $\hat{\phi}^N$ can fit the unseen data well. 

To develop our theory, we first make the following assumptions on the approximation error of $V$:
\begin{assumption}
Let $\hat{u}$ be the solution to (\ref{eq:energy_general}). There exist a neural network $\phi\in V$ such that
\begin{equation*}
    \Vert\phi-\hat{u}\Vert_{H^1(\rho_x)}^2:=\mathbb{E}_{\rho_x}\left(|\phi-\hat{u}|^2+|\nabla\phi-\nabla \hat{u}|^2\right)\leq C_1^2\frac{\Vert\hat{u}\Vert_*^2}{m^\gamma}
\end{equation*}
where $C_1$ is a constant, $\Vert\cdot\Vert_*$ is the norm associated with a certain space, $m$ is a constant related to the number of parameters in $\phi$, and $\gamma$ is a constant order.
\end{assumption}
The second assumption we need is on the quadrature error:
\begin{assumption}
Let $\phi$ be any neural network in $V$ and $\{x_i\}_{i=1}^N$ be a set of points i.i.d. sampled from $\rho_x$. Given any $0<\epsilon<1$, we assume that the quadrature error satisfies
\begin{equation*}
    \sup_{\phi\in V}|L^R(\phi)-L^R_N(\phi)|\leq\frac{C_2(V,m,\epsilon)}{\sqrt{N}}
\end{equation*}
with probability at least $1-\epsilon$, where $C_2$ relies on the set $V$, $m$ and $\epsilon$. Besides, $C_2$ may also depends on some other factors like $d$.
\end{assumption}
Such an inequality can be obtained using some statistical learning theories like the Rademacher complexity. In the following part, we will bound the generalization error using the approximation error and the quadrature error.

\subsection{Error analysis of TV regularized networks}
We consider a general expected loss with TV regularization:
\begin{equation}
    L^{TV}(u)=\mathbb{E}_{\rho_x}\left(D(u(x),y(x))+|\nabla u(x)|\right), \quad u\in U \label{eq:TV_energy_1}.
\end{equation}
We further assume that $D(u,y)$ is $\lambda$-Lipchitz continuous with respect to $u$, i.e,:
\begin{equation*}
    |D(u_1(x),y(x))-D(u_2(x),y(x))|\leq\lambda|u_1(x)-u_2(x)|
\end{equation*}
for any $u_1, u_2$ in $U$ and $x\in\Omega$. Let $V\subset U$ be a set of neural networks with finite number of parameters. In practice, we train the networks with the empirical loss:
\begin{equation}
     L^{TV}_N(\phi)=\frac{1}{N}\sum_{i=1}^N\left(D(\phi(x_i;\theta),y(x_i))+|\nabla\phi(x_i;\theta)|\right), \phi\in V \label{eq:TV_energy_3}.
\end{equation}
For this TV regularized deep learning problem, we give a bound for the generalization error in the following theorem.
\begin{theorem} \label{theorem:GE_TV}
Let $\hat{u}\in U$ be the minimizer of $L^{TV}(u)$, $\hat{\phi}\in V$ be a $\delta$ minimizer of $L^{TV}(\phi)$, and $\hat{\phi}^N\in V$ be a $\delta$ minimizer of $L^{TV}_N(\phi)$. Then

(a) \begin{equation*}
    0\leq L^{TV}(\hat{\phi})-L^{TV}(\hat{u}) \leq (\lambda+1)C_1\frac{\Vert\hat{u}\Vert_*}{m^{\gamma/2}} + \delta.
\end{equation*}

(b) Given $0<\epsilon<1$,

\begin{equation*}
    0\leq L^{TV}(\hat{\phi}^N)-L^{TV}(\hat{u})\leq (\lambda+1)C_1\frac{\Vert\hat{u}\Vert_*}{m^{\gamma/2}} + \frac{2C_2(V,m,\epsilon)}{\sqrt{N}} + 2\delta
\end{equation*}
with probability at least $1-\epsilon$ over the choice of $x_i$. 
\end{theorem}
\begin{proof}
From Assumption 1, we know there exist a network $\hat{\psi}\in V$ such that $\Vert\psi-\hat{u}\Vert_{H^1}^2\leq C_1^2\frac{\Vert\hat{u}\Vert_*^2}{m^\gamma}$. Then,
\begin{align*}
    & L^{TV}(\hat{\phi})-L^{TV}(\hat{u}) \\
    & \leq L^{TV}(\hat{\psi})-L^{TV}(\hat{u}) + \delta\\
    & \leq \lambda \mathbb{E}_{\rho_x}(|\hat{\psi}(x;\theta)-\hat{u}(x)|) +   \mathbb{E}_{\rho_x}(|\nabla\hat{\psi}(x;\theta)-\nabla\hat{u}(x)|)  + \delta\\
    & \leq \lambda\sqrt{\mathbb{E}_{\rho_x}(|\hat{\psi}(x;\theta)-\hat{u}(x)|^2)}+\sqrt{\mathbb{E}_{\rho_x}(|\nabla\hat{\psi}(x;\theta)-\nabla\hat{u}(x)|^2)} + \delta\\
    & \leq (\lambda+1)C_1\frac{\Vert\hat{u}\Vert_*}{m^{\gamma/2}} + \delta.
\end{align*}
This completes the proof of (a).

From assumption 2, given $0<\epsilon<1$, we have
\begin{equation*}
    P\left(|L^{TV}_N(\phi)-L^{TV}(\phi)|\leq \frac{C_2(V,m,\epsilon)}{\sqrt{N}}\right)\geq 1-\epsilon
\end{equation*}
for any $\phi\in V$.
Then,
\begin{align*}
    L^{TV}(\hat{\phi}^N)-L^{TV}(\hat{u}) & \leq L^{TV}_N(\hat{\phi}^N)-L^{TV}(\hat{u}) + \frac{C_2(V,m,\epsilon)}{\sqrt{N}} \\
    & \leq  L^{TV}_N(\hat{\phi})-L^{TV}(\hat{u}) + \frac{C_2(V,m,\epsilon)}{\sqrt{N}}  + \delta\\
    & \leq L^{TV}(\hat{\phi})-L^{TV}(\hat{u}) + \frac{2C_2(V,m,\epsilon)}{\sqrt{N}}  + \delta\\
    & \leq (\lambda+1)C_1\frac{\Vert\hat{u}\Vert_*}{m^{\gamma/2}} + \frac{2C_2(V,m,\epsilon)}{\sqrt{N}}  + 2\delta.
\end{align*}
$\hfill\square$
\end{proof}
This theorem gives an upper bound for the generalization error between the numerical solution $\hat{\phi}^N$ found by solving the empirical loss $L^{TV}_N$ and the true solution $\hat{u}$. We see that the error bound consists of three terms. The first term $(\lambda+1)C_1\frac{\Vert \hat{u}\Vert_*}{m^{\gamma/2}}$ corresponds to the approximation error given in Assumption 1, the second term $\frac{2C_2(V,m,\epsilon)}{\sqrt{N}}$ corresponds to the quadrature error given in Assumption 2, and the last term $2\delta$ corresponds to the training error. To reduce the generalization error, we need increase both the number of parameters and training samples. Meanwhile, the empirical loss is needed to be solved sufficiently well. Notice that the training samples are randomly generated, so the generalization error bound holds with certain probability.

\subsection{Error analysis of Tikhonov regularized networks}
Now we consider a Tikhonov regularized expected loss:
\begin{equation}
    L^{Tik}(u)=\mathbb{E}_{\rho_x}\left(D(u(x),y(x))+ |\nabla u(x)|^2\right),\quad u\in U \label{eq:TK_energy_1}.
\end{equation}
Suppose $D(u(x),y(x))$ is also $\lambda$-Lipchitz continuous
\begin{equation*}
|D(u_1(x),y(x))-D(u_2(x),y(x))|\leq \lambda|u_1(x)-u_2(x)|.
\end{equation*}
The empirical loss is given as:
\begin{equation}
L^{Tik}_N(\phi)=\frac{1}{N}\sum_{i=1}^N\left(D(\phi(x_i),y(x_i))+|\nabla \phi(x_i)|^2\right), \quad \phi\in V \label{eq:TK_energy_3}.
\end{equation}
Then, we can give the error estimate in the following theorem. The proof is very similar to Theorem \ref{theorem:GE_TV}.
\begin{theorem} \label{theorem:GE_TK}
Let $\hat{u}\in U$ be the minimizer of $L^{Tik}(u)$, $\hat{\phi}\in V$ be a $\delta$ minimizer of $L^{Tik}(\phi)$, and $\hat{\phi}^N\in V$ be a $\delta$ minimizer of $L^{Tik}_N(\phi)$. Then

(a) \begin{equation*}
    0\leq L^{Tik}(\hat{\phi})-L^{Tik}(\hat{u}) \leq (\lambda+2\Vert\nabla\hat{u}\Vert_2)C_1\frac{\Vert\hat{u}\Vert_*}{m^{\gamma/2}}+C_1^2\frac{\Vert\hat{u}\Vert^2_*}{m^{\gamma}} + \delta
\end{equation*}
where the $L_2$ norm is defined as $\Vert\nabla\hat{u}\Vert_2=\sqrt{\mathbb{E}_{\rho_x}(|\nabla\hat{u}|^2)}$.

(b) For any $0<\epsilon<1$,
\begin{align*}
    0&\leq L^{Tik}_N(\hat{\phi}^N)-L^{Tik}(\hat{u})\\
    &\leq  (\lambda+2\Vert\nabla\hat{u}\Vert_2)C_1\frac{\Vert\hat{u}\Vert_*}{m^{\gamma/2}}+C_1^2\frac{\Vert\hat{u}\Vert_*^2}{m^\gamma} + \frac{2C_2(V,m,\epsilon)}{\sqrt{N}} + 2\delta
\end{align*}
with probability at least $1-\epsilon$ over the choice of $x_i$. 
\end{theorem}
\begin{proof}
From Assumption 1, there exist a neural network $\hat{\psi}\in V$ such that $\Vert\psi-\hat{u}\Vert_{H^1}^2\leq C_1^2\frac{\Vert\hat{u}\Vert^2_*}{m^{\gamma}}$. Then,
\begin{align*}
    & L^{Tik}(\hat{\phi})-L^{Tik}(\hat{u}) \\
     \leq& L^{Tik}(\hat{\psi})-L^{Tik}(\hat{u}) + \delta\\
     \leq& \lambda\mathbb{E}_{\rho_x}(|\hat{\psi}-\hat{u}|) + \mathbb{E}_{\rho_x}(|\nabla\hat{\psi}-\nabla u|^2)+2\mathbf{E}_{\rho_x}(\langle\nabla\hat{\psi}-\nabla \hat{u},\nabla \hat{u}\rangle) + \delta\\
     \leq& (\lambda+2\Vert\nabla\hat{u}\Vert_2)C_1\frac{\Vert\hat{u}\Vert_*}{m^{\gamma/2}}+C_1^2\frac{\Vert\hat{u}\Vert^2_*}{m^{\gamma}} + \delta, 
\end{align*}
which proves (a).

From Assumption 2, we have
\begin{equation*}
    |L^{Tik}_N(\hat{\phi})-L^{Tik}(\hat{\phi})|\leq\frac{C_2(V,m,\epsilon)}{\sqrt{N}}
\end{equation*}
and
\begin{equation*}
    |L^{Tik}_N(\hat{\phi}^N)-L^{Tik}(\hat{\phi}^N)|\leq \frac{C_2(V,m,\epsilon)}{\sqrt{N}}.
\end{equation*}
with probability at least $1-\epsilon$ over the choice of $x_i$. Then,
\begin{align*}
     & L^{Tik}(\hat{\phi}^N)-L^{Tik}(\hat{u}) \\
     \leq & L^{Tik}_N(\hat{\phi}^N)-L^{TV}(\hat{u}) + \frac{C_2(V,m,\epsilon)}{\sqrt{N}} \\
    \leq & L^{Tik}_N(\hat{\phi})-L^{Tik}(\hat{u}) + \frac{C_2(V,m,\epsilon)}{\sqrt{N}}  + \delta\\
     \leq & L^{Tik}(\hat{\phi})-L^{Tik}(\hat{u}) + \frac{2C_2(V,m,\epsilon)}{\sqrt{N}}  + \delta\\
    \leq & (\lambda+2\Vert\nabla\hat{u}\Vert_2)C_1\frac{\Vert\hat{u}\Vert_*}{m^{\gamma/2}}+C_1^2\frac{\Vert\hat{u}\Vert^2_*}{m^{\gamma}} + \frac{2C_2(V,m,\epsilon)}{\sqrt{N}} + 2\delta.
\end{align*}
$\hfill\square$
\end{proof}
This theorem gives the bound on the generalization error between the neural network solution $\hat{\phi}^N$ found by solving $L^{Tik}_N$ and the true solution $\hat{u}$. Similar to the results in Theorem \ref{theorem:GE_TV}, the error bounds also consists of three parts which correspond to the approximation error, quadrature error and training error respectively.

If the data term $D(u,y)$ is not Lipchitz continuous but H$\Ddot{\text{o}}$lder continuous with component 2, we can also derive a generalization bound in the same way. 
Besides, when $D(u,y)=|u-y|^2$, we have the following error estimates. 

\begin{theorem} \label{theorem:SE_TK}
Suppose $D(u(x),y(x))=|u(x)-y(x)|^2$ and $y(x)$ is bounded. Let $\hat{u}\in U$ be the unique minimizer of $L^{Tik}(u)$, $\hat{\phi}\in V$ be a $\delta$ minimizer of $L^{Tik}(\phi)$, and $\hat{\phi}^N\in V$ be a $\delta$ minimizer of $L^{Tik}_N(\phi)$. Then

(a) \begin{equation*}
    \Vert\hat{\phi}-\hat{u}\Vert_{H^1}^2 \leq C_1^2\frac{\Vert\hat{u}\Vert_*^2}{m^\gamma}+\delta.
\end{equation*}

(b) For any $0<\epsilon<1$, 
\begin{equation*}
    \Vert\hat{\phi}^N-\hat{u}\Vert_{H^1}^2 \leq C_1^2\frac{\Vert\hat{u}\Vert_*^2}{m^\gamma} + \frac{2C_2(V,m,\epsilon)}{\sqrt{N}}+2\delta
\end{equation*}
with probability at least $1-\epsilon$ over the choice of $x_i$. 
\end{theorem}
\begin{proof}
Since $\hat{u}$ is the solution to (\ref{eq:TK_energy_1}), we have
\begin{equation*}
    \mathbb{E}_{\rho_x}((\hat{u}(x)-y(x))v(x)+ \langle \nabla\hat{u}(x),\nabla v(x)\rangle)=0
\end{equation*}
for any $v\in U$. Thus, $L^{Tik}(v)-L^{Tik}(\hat{u}) = \Vert v-\hat{u}\Vert_{H^1}^2$ for $\forall v\in U$. Let $\hat{\psi}\in V$ be a two layer network such that $\Vert\hat{\psi}-\hat{u}\Vert_{H^1}^2\leq C_1^2\frac{\Vert\hat{u}\Vert_*^2}{m^\gamma}$. Then
\begin{align*}
    \Vert\hat{\phi}-\hat{u}\Vert_{H^1}^2 & = L^{Tik}(\hat{\phi})-L^{Tik}(\hat{u}) \\
    & \leq L^{Tik}(\hat{\psi})-L^{Tik}(\hat{u}) + \delta\\
    &  =\Vert\hat{\psi}-\hat{u}\Vert_{H^1}^2 + \delta\\
    & \leq C_1^2\frac{\Vert\hat{u}\Vert_*^2}{m^\gamma} + \delta
\end{align*}
which proves (a).
Using the same argument with the previous proofs, we can show that
\begin{equation*}
    \Vert\hat{\phi}^N-\hat{u}\Vert_{H^1}^2\leq C_1^2\frac{\Vert\hat{u}\Vert_*^2}{m^\gamma} + \frac{2C_2(V,m,\epsilon)}{\sqrt{N}} + 2\delta
\end{equation*}
with probability at least $1-\epsilon$.
$\hfill\square$
\end{proof}
This theorem is a special case of Tikhonov regularized problem. The generalization error equals to the Sobolev norm error when the data term has specific form. Let the space $U$ be the Sobolev space $H^1(S^d)$. If we further assume the probability distribution $\rho_x$ is an uniform distribution, then $\rho_x$ coincides with the Lebesgue measure. The original problem (\ref{eq:TK_energy_1}) is written as
\begin{equation*}
    \min_{u\in U} L^{Tik}(u)= \min_{u\in H^1(S^d)}\int_{S^d}|u(x)-y(x)|^2+|\nabla u(x)|^2 dx
\end{equation*}
which is the weak formulation of the PDE
\begin{align}
    u-\Delta u=y & \text{ in } S^d, \label{equation:PDE}\\
    \frac{\partial u}{\partial\mathbf{n}}=0 & \text{ on } \partial S^d.
\end{align}
Therefore, Theorem \ref{theorem:SE_TK} can also be veiwed as the generalization error estimate for the second order PDE $u-\Delta u=y$ with Neumann boundary condition. A similar error estimate for this type of PDE is also derived in a recent work \cite{lu2021priori}.

\section{Error analysis of two-layer neural networks} \label{sec4}
We consider two variational problems
\begin{equation*}
    \min_{u\in BV(\Omega)}L^{TV}(u)=\int_{S^d}u(x)^2-2u(x)y(x)+|\nabla u(x)|\ dx
\end{equation*}
and
\begin{equation*}
    \min_{u\in H^1(\Omega)}L^{Tik}(u)=\int_{S^d}u(x)^2-2u(x)y(x)+|\nabla u(x)|^2\ dx,
\end{equation*}
where $y$ is uniformly bounded by $Y>0$ for $x\in\Omega$.
The first problem is the ROF model and the second problem is the variational formulation of the second order PDE (\ref{equation:PDE}). We further assume that the solution to these two problems belong to the Barron space as well. In fact, for ROF models, this condition may only hold in some special cases, since the BV space contains functions with jumps and functions in Barron spaces are continuous. In this section, we will analyze the generalization error of these two problems solved by two-layer networks. 
Define the set of two-layer networks with bounded parameters by
\begin{align*}
    \mathcal{F}^{m}_{\sigma}(B):=&\bigg\{\phi(x;\theta)=c+\frac{1}{m}\sum_{k=1}^ma_k\sigma(\omega_k\cdot x+b_k)\bigg|\\
    &|c|\leq 2B, |a_k|\leq 16B, |\omega_k|_1= 1, |b_k|\leq 1 \text{ for } k=1,\dots,m\bigg\},
\end{align*}
where $\sigma$ is the activation function, $B$ is a positive constant and $m$ is the width of the network. The boundedness of parameters is necessary for the estimation of Rademacher complexity in the proofs of following lemmas. Moreover, we define the space of Barron function as: 
\begin{definition}
For a function $f$ defined on $\Omega$, if there exist an extension of $f$ to $\mathbb{R}^d$ with Fourier transform $\tilde{f}(\omega)$ such that
\begin{equation}
    f(x)=\int_{\mathbb{R}^d} e^{i\omega\cdot x}\tilde{f}(\omega)d\omega=\int_{\mathbb{R}^d} e^{i\omega\cdot x}e^{i\zeta(\omega)}|\tilde{f}(\omega)|d\omega,\quad x\in\Omega, \label{eq:barron_function}
\end{equation}
and
\begin{equation*}
    \int_{\mathbb{R}^d}(1+|\omega|_1)^s|\tilde{f}(\omega)|d\omega<+\infty,
\end{equation*}
we say $f\in\mathbb{R}$ belongs to the Barron space $\B^s(\Omega)$. $\zeta(\omega)$ here denotes the phase of $\tilde{f}(\omega)$ and $|\tilde{f}(\omega)|$ is the magnitude. The associated Barron norm of $f$ is defined by 
$$\Vert f\Vert_{\B^s(\Omega)}:=\inf_{\tilde{f}}\int_{\mathbb{R}^d}(1+|\omega|_1)^s|\tilde{f}(\omega)|d\omega,$$ 
\end{definition}
where the inf is taken over all the extension of $f$ such that (\ref{eq:barron_function}) holds.
When $s=2$, we simply denote $\B^s$ as $\B$. We then give the approximation error of the set $\mathcal{F}^m_{\relu}(B)$ to the function in $\B$ in the sense of $H^1$ norm, where $\relu(z)=\max\{0,z\}$.
\begin{lemma}\label{theorem:H1ReLU}
For any function $f\in\B(\Omega)$ and $m\in\mathbb{N}^+$, there exist a two-layer ReLU network $f_m\in\mathcal{F}^m_{\relu}(\Vert f\Vert_{\B}+\eta)$ such that
\begin{equation*}
    \Vert f-f_m\Vert_{H^1(\Omega)}\leq \frac{\sqrt{1345}(\Vert f\Vert_{\B(\Omega)}+\eta)}{\sqrt{m}}
\end{equation*}
for any $\eta>0$.
\end{lemma}
Neural networks with ReLU activations are very popular in machine learning applications. However, we may require the networks to be secondly differentiable sometimes. Instead, we can consider a smooth approximation to ReLU, named SoftPlus function \cite{lu2021priori},
\begin{equation*}
    \SP_\tau(z)=\frac{1}{\tau}\log(1+e^{\tau z}).
\end{equation*}
$\SP_\tau$ converges uniformly to ReLU when $\tau\longrightarrow+\infty$.
\begin{lemma}\label{theorem:H1SP}
For any function $f\in\B(\Omega)$ and $m\in\mathbb{N}^+/\{1\}$, by setting $\tau=\sqrt{m}$, there exist a two-layer SP network $f_m\in\mathcal{F}^m_{\SP_{\tau}}(\Vert f\Vert_{\B}+\eta)$ such that
\begin{equation*}
    \Vert f-f_m\Vert_{H^1(\Omega)}\leq \frac{\Vert f\Vert_{\B}+\eta}{\sqrt{m}}\left(24\log(m)+65\right)
\end{equation*}
for any $\eta>0$.
\end{lemma}
\begin{lemma} \label{lemma:quadratureerror}
Let $(x_1,\dots,x_N)$ be a set of i.i.d variables sampled from $\Omega$. Then, given any $\epsilon\in(0,1)$, the probability of
\begin{equation*}
    \sup_{\phi_m\in\mathcal{F}^m_{\SP_{\tau}}(B)}\left|L^{TV}_N(\phi_m)-L^{TV}(\phi_m)\right|\leq 2\frac{\sqrt{m}}{\sqrt{N}}(Z_1+Z_2\sqrt{\log(m)})+M_1\frac{\sqrt{-2\log(\epsilon)}}{\sqrt{N}}
\end{equation*}
is at least $1-\epsilon$. Similarly, given any $\epsilon\in(0,1)$, the probability of
\begin{equation*}
     \sup_{\phi_m\in\mathcal{F}^m_{\SP_{\tau}}(B)}\left|L^{Tik}_N(\phi_m)-L^{Tik}(\phi_m)\right|\leq 2\frac{\sqrt{m}}{\sqrt{N}}(Z_1+Z_3\sqrt{\log(m)})+M_2\frac{\sqrt{-2\log(\epsilon)}}{\sqrt{N}}
\end{equation*}
is at least $1-\epsilon$ as well. Moreover, $\mathcal{F}^m_{\SP_{\tau}}(B)$ is the set of two-layer softplus network with bounded parameters, $Z_1$ is a constant depends on $B$, $d$ and $Y$; $Z_2$, $Z_3$ are constants depend on $B$ and $d$; $M_1$, $M_2$ are constants depend on $B$ and $Y$.
\end{lemma}
The proofs of Lemma \ref{theorem:H1ReLU} to \ref{lemma:quadratureerror} are deferred to the appendix.
Combing the approximation error in Lemma \ref{theorem:H1SP} and the quadrature error in Lemma \ref{lemma:quadratureerror}, we can give the generalization error for the two problems using the framework we developed before.
\begin{theorem}
Let $\hat{u}=\arg\min_{u\in BV(\Omega)}L^{TV}(u)$, and $\mathcal{F}^m_{\SP_{\tau}}(B)$ be a set of softplus two-layer networks with $B>0$ and $\tau=\sqrt{m}$. We assume the solution $\hat{u}$ belongs to the Barron space $\B(\Omega)$ and $B>\Vert\hat{u}\Vert_{\B(\Omega)}$. Suppose $\hat{\phi}^N\in\mathcal{F}^m_{\SP_{\tau}}(B)$ be a $\delta$ minimizer of $L^{TV}_N$. Then
\begin{align*}
    L^{TV}(\hat{\phi}^N)-L^{TV}(\hat{u})\leq&(2Y+49B+\Vert\hat{u}\Vert_{L^2(\Omega)}+1)\frac{\Vert f\Vert_{\B}}{\sqrt{m}}\left(24\log(m)+65\right)\\
    &\quad +4\frac{\sqrt{m}}{\sqrt{N}}(Z_1+Z_2\sqrt{\log(m)})+2M_1\frac{\sqrt{-2\log(\epsilon)}}{\sqrt{N}}+2\delta,
\end{align*}
with probability at least $1-\epsilon$ over the choice of $(x_1,\dots,x_N)$, where $Y$ is the uniform bound of $y$; The constants $Z_1$, $Z_2$ and $M_1$ are defined in Lemma \ref{lemma:quadratureerror}.
\end{theorem}
\textbf{Remark.} The functions in $BV(\Omega)$ may not necessarily belong to $H^1(\Omega)$. However, since we assume the solution belongs to the Barron space, the gradient of $\hat{u}$ exists \cite{barron1993universal}.
\begin{proof}
For any $\phi\in\mathcal{F}^m_{\SP_{\tau}}(B)$, we have $\phi(x)\leq B+16B\SP_{\tau}(2)\leq B+16B(2+\frac{1}{\tau})\leq 49B$ for $x\in\Omega$. Then,
\begin{equation*}
    |\phi-y|^2-|\hat{u}-y|^2=\langle\phi-\hat{u},\phi+\hat{u}-2y\rangle\leq\Vert\phi-\hat{u}\Vert_{L^2(\Omega)}((2Y+49B)|S^d|+\Vert\hat{u}\Vert_{L^2(\Omega)}).
\end{equation*}
We require the data term to be Lipschitz continuous in Theorem \ref{theorem:GE_TV}. However, in the proof of the theorem, we just need $D(\phi,y)-D(\hat{u},y)\leq\lambda\Vert\phi-\hat{u}\Vert_{L^2(\Omega)}$ for some $\lambda$ and any $\phi\in\mathcal{F}^m_{\SP_{\tau}}(B)$. We apply Theorem \ref{theorem:GE_TV} with $\lambda=2Y+49B+\Vert\hat{u}\Vert_{L^2(\Omega)}$ and let $\eta<B-\Vert \hat{u}\Vert_{\B(\Omega)}$ goes to 0. Then we get the generalization error. 
$\hfill\square$
\end{proof}
\begin{theorem}
Let $\hat{u}=\arg\min_{u\in H^1(\Omega)}L^{Tik}(u)$, and $\mathcal{F}^m_{\SP_{\tau}}(B)$ be a set of softplus two-layer networks with $B>0$ and $\tau=\sqrt{m}$. We assume the solution $\hat{u}$ belongs to the Barron space $\B(\Omega)$ and $B>\Vert\hat{u}\Vert_{\B(\Omega)}$. Suppose $\hat{\phi}^N\in\mathcal{F}^m_{\SP_{\tau}}(B)$ be a $\delta$ minimizer of $L^{Tik}_N$. Then
\begin{align*}
    \Vert\hat{\phi}^N-\hat{u}\Vert_{H^1(\Omega)}^2\leq&\frac{\Vert f\Vert_{\B}^2}{m}\left(24\log(m)+65\right)^2\\
    &\quad +4\frac{\sqrt{m}}{\sqrt{N}}(Z_1+Z_3\sqrt{\log(m)})+2M_2\frac{\sqrt{-2\log(\epsilon)}}{\sqrt{N}}+2\delta,
\end{align*}
with probability at least $1-\epsilon$ over the choice of $(x_1,\dots,x_N)$, where The constants $Z_1$, $Z_3$ and $M_2$ are defined in Lemma \ref{lemma:quadratureerror}.
\end{theorem}
\begin{proof}
The result is obtained by applying Theorem \ref{theorem:SE_TK}, Lemma \ref{theorem:H1SP} and Lemma \ref{lemma:quadratureerror}, and let $\eta<B-\Vert \hat{u}\Vert_{\B(\Omega)}$ goes to 0.
$\hfill\square$
\end{proof}
\textbf{Remark.} In these two theorems, we require the original solution $\hat{u}$ belongs to the Barron space. In \cite{barron1993universal}, we know functions with sufficiently high order square integrable derivatives belong to the Barron space. Besides, some regularity results \cite{lu2021priori} of certain types of PDEs can also guarantee $\hat{u}$ belongs to Barron type spaces. As we mentioned before, this condition is difficult for the ROF model to satisfy. To overcome this limitation, we need approximation results of neural networks to more general functions.

\section{Numerical experiments}\label{sec5}
\subsection{Experiment setup}
In this section, we will compare the gradient-based methods with other explicit regularization methods on image classification datasets. All the neural networks will be tested on both the original test samples and adversarial samples. Generally, there are two types of adversarial attack methods to generate adversarial samples: white box attack and black box attack. The white box attacks, like the fast gradient sign method (FGSM) \cite{goodfellow2015explaining} (Algorithm \ref{algo:FGSM}) and the projected gradient descent method (PGD) \cite{madry2018towards} (Algorithm \ref{algo:PGD}), use the information of the network structure and parameters to generate adversarial samples, and the black box attacks, like the One Pixel attack \cite{su2019one} and the transfer attack \cite{papernot2017practical} (Algorithm \ref{algo:transfer}), can only access the output of networks. 
Since the FGSM is a special case of the PGD attack, we only use the PGD attack for evaluating the robustness to white box attacks in this section. 

The regularization methods we implement are the total variation (TV), Tikhonov regularization (Tik), graphical total variation (GTV), graphical Tikhonov regularization (GTik), L2 penalty (L2), adversarial training (AT), and adversarial logit pairing (ALP). The standard neural network training without any regularization is set as the baseline method.
\begin{algorithm}
\caption{Fast gradient sign method (FGSM)}\label{algo:FGSM}
\begin{algorithmic}
\REQUIRE Targeted network $\phi$, a sample $x$ to be attacked and a step size $\delta$
\STATE Compute $dx=\text{sign}(\nabla_x D(\phi(x;\theta),y(x)))$
\RETURN $x+\delta * dx$
\end{algorithmic}
\end{algorithm}

\begin{algorithm}
\caption{Projected gradient descent method (PGD)}\label{algo:PGD}
\begin{algorithmic}
\REQUIRE Targeted network $\phi$, a sample $x$ to be attacked, the step size $\delta$, the perturbation bound $\epsilon$, and the number of iterations $N_{itr}$.
\STATE $x_0=x$
\FOR{i=1:$N_{itr}$}
\STATE $x_i=x_{i-1}+\delta*\text{sign}(\nabla_x D(\phi(x;\theta),y(x)))$
\STATE $dx=\text{Clip}(x_i-x,\min=-\epsilon,\max=\epsilon)$
\STATE $x_i=\text{Clip}(x+dx,\min=0,\max=1)$
\ENDFOR
\RETURN $x_N$
\end{algorithmic}
\end{algorithm}

\begin{algorithm}
\caption{Tranfer attack}\label{algo:transfer}
\begin{algorithmic}
\REQUIRE Targeted network $\phi$, a small set of training data $\tilde
{x}$, a sample $x$ to be attacked, the step size $\delta$, the perturbation bound $\epsilon$, and the number of iterations $N_{itr}$.
\STATE 1. Construct a substitue network $\tilde{\phi}$.
\STATE 2. Train $\tilde{\phi}$ with samples $\tilde{x}$ and labels $\arg\max(\phi(\tilde{x}))$.
\STATE 3. Generate the adversarial sample for $x$ by using the PGD algorithm on $\tilde{\phi}$.

\end{algorithmic}
\end{algorithm}

\subsection{Training with the MNIST dataset}
We first evaluate the methods on the MNIST dataset which consists of various gray scale hand-written digits images. We randomly generate 100 images from the MNIST dataset as the training set and another 100 images as the validation set which is used to determine the weight $\alpha$. The $\alpha$ we choose for different methods are listed in Table \ref{table:parameters_MNIST}. Generally, the performance of the networks are not very sensitive to the choice of $\alpha$ within certain ranges. The full MNIST test set is used to evaluate the test accuracy and 1,000 randomly selected test samples are used to evaluate the robustness. The neural network used here is a simple convolutional neural network (CNN) whose architecture is shown in Appendix. When training the network, we use the negative log likelihood loss (Pytorch NLLLoss) as the data term and the gradient terms in the regularization is directly computed by the back propagation. To reduce the variance of our results, we
do 50 independent trials and compute the averaged accuracy. During the training stage, we use the standard Adam algorithm \cite{Kingma2015Adam} with learning rate 0.001 and the number of epochs is 1,000. For the $L2$ penalty method, we implement it by the AdamW optimizer \cite{loshchilov2018decoupled}.

\begin{table}
\centering
\begin{tabular}{@{}lllllll@{}}
\toprule
TV \quad\quad    & Tik \quad\quad   & GTV \quad\quad   & GTik \quad\quad  & L2 \quad\quad & AT \quad\quad & ALP \quad\quad \\ \midrule
0.005 & 0.001 & 0.001 & 0.005 & 0.05 & 1. & 1.  \\ \bottomrule
\end{tabular}
\caption{\textit{The weight of regularization term for each method in the MNIST experiments.}}
\label{table:parameters_MNIST}\vspace{-2em}
\end{table}

The prediction accuracy are shown in Table \ref{table:accuracy_MNIST_100}. All the regularization methods can improve the prediction accuracy on the test set. The highest accuracy is obtained by the graph-based methods. In terms of the robustness, the ALP method is the most robust to the white box PGD attack. This is not surprising because the adversarial samples for ALP training are generated by the same attack method. For the black box one pixel attack and transfer attack, the TV and GTV methods are the most robust respectively. 

\begin{table}
\small
\centering
\begin{tabular}{lllllllll}
\hline
\multicolumn{1}{l|}{} & baseline   & TV               & Tik     & GTV     & GTik  & L2   & AT      & ALP              \\ \hline
\multicolumn{1}{l|}{no attack}  & 78.45\% &  83.69\% &  83.44\% &  \textbf{85.88\%} &  85.33\% & 78.77\% &  81.86\% & 82.6\% \\ 
\multicolumn{1}{l|}{PGD}           & 15.70\% & 39.88\%          & 38.42\% & 19.98\% & 16.81\% &23.46\% & 32.25\% & \textbf{47.89\%} \\
\multicolumn{1}{l|}{One Pixel}     & 37.78\% & \textbf{66.76\%} & 65.85\% & 36.37\% & 37.01\% & 37.32\% & 46.32\% & 48.77\%          \\
\multicolumn{1}{l|}{Transfer}       & 38.55\% & 61.19\% & 58.68\% & \textbf{64.77\%} & 64.17\% &42.41\% & 50.58\% & 57.35\%\\
\hline 
\end{tabular}
\caption{\textit{Prediction accuracy of different methods on original samples and adversarial samples (trained with 100 MNIST samples). The first row shows the prediction accuracy on the original test samples without any adversarial attack. The last three rows shows the prediction accuracy on the adversarial samples generated by different attack algorithms.}}
\label{table:accuracy_MNIST_100}\vspace{-2em}
\end{table}

We also train the network on 10,000 images randomly selected from the MNIST dataset. All the hyperparameters remain the same with the previous experiment. The prediction accuracy, which is the average of 5 independent trails, is listed in Table \ref{table:accuracy_MNIST_10000}. The accuracy of all the methods are very close when the number of training samples is sufficiently large. The ALP method is still the most robust to the PGD attack, and the gradient-based methods are the most robust to the black box attacks. We also notice that the graph-based methods are very unrobust with respect to PGD and one pixel attacks.

\begin{table}
\small
\centering
\begin{tabular}{lllllllll}
\hline
\multicolumn{1}{l|}{} & baseline  & TV               & Tik              & GTV     & GTik  &L2  & AT      & ALP              \\ \hline
\multicolumn{1}{l|}{no attack}   & 98.11\% & 98.23\% & 98.41\% & 98.65\% & \textbf{98.80\%} &98.64\% & 98.27\% & 98.10\% \\
\multicolumn{1}{l|}{PGD}            & 20.76\% & 84.00\%          & 83.94\%          & 15.46\%  & 25.16\% & 37.36\% & 27.08\% & \textbf{90.14\%} \\
\multicolumn{1}{l|}{One Pixel}     & 77.06\% & 91.76\%          & \textbf{91.90\%} & 25.18\% & 20.80\% &87.44\% & 76.94\% & 82.04\%          \\
\multicolumn{1}{l|}{Transfer}      & 78.18\% & 92.96\% & \textbf{93.40\%}          & 89.34\% &90.54\% & 69.18\% & 81.38\% & 91.92\%          \\
\hline 
\end{tabular}
\caption{\textit{Prediction accuracy of different methods on original samples and adversarial samples (trained with 10,000 MNIST samples). The first row shows the prediction accuracy on the original test samples without any adversarial attack. The last three rows shows the prediction accuracy on the adversarial samples generated by different attack algorithms.}}
\label{table:accuracy_MNIST_10000}\vspace{-3em}
\end{table}

\subsection{Training with the Fashin MNIST (FMNIST) dataset}
We also conduct experiments on the FMNIST dataset and the setup is the same with MNIST experiments. The weights of regularization for each method are listed in Table \ref{table:parameters_FMNIST_100} and Table \ref{table:parameters_FMNIST_10000}. The results are listed in Table \ref{table:accuracy_FMNIST_100} and Table \ref{table:accuracy_FMNIST_10000}.
\begin{table}
\small
\centering
\begin{tabular}{@{}lllllll@{}}
\toprule
TV \quad\quad    & Tik \quad\quad   & GTV \quad\quad   & GTik \quad\quad  & L2 \quad\quad & AT \quad\quad & ALP \quad\quad \\ \midrule
0.01 & 0.001 & 0.01 & 0.05 & 0.1 & 1. & 0.1  \\ \bottomrule
\end{tabular}
\caption{\textit{The weight of regularization term for each method in the FMNIST experiments (trained with 100 samples).}}
\label{table:parameters_FMNIST_100}\vspace{-3em}
\end{table}

\begin{table}
\small
\centering
\begin{tabular}{@{}lllllll@{}}
\toprule
TV \quad\quad    & Tik \quad\quad   & GTV \quad\quad   & GTik \quad\quad  & L2 \quad\quad & AT \quad\quad & ALP \quad\quad \\ \midrule
0.01 & 0.001 & 0.001 & 0.001 & 0.1 & 1. & 0.1  \\ \bottomrule
\end{tabular}
\caption{\textit{The weight of regularization term for each method in the FMNIST experiments (trained with 10,000 samples).}}
\label{table:parameters_FMNIST_10000}\vspace{-2em}
\end{table}

\begin{table}
\small
\centering
\begin{tabular}{lllllllll}
\hline
\multicolumn{1}{l|}{regularization} & baseline   & TV               & Tik     & GTV     & GTik  & L2   & AT      & ALP              \\ \hline
\multicolumn{1}{l|}{no attack}  & 67.91\% & 69.40\% & 69.03\% & 71.87\% & \textbf{72.30\%} & 67.87\% & 68.55\% & 69.00\% \\ 
\multicolumn{1}{l|}{PGD}           & 18.61\% & 35.01\% & 35.49\% & 4.40\%  & 9.38\%  & 19.12\% & 22.37\% & \textbf{51.42\%} \\
\multicolumn{1}{l|}{One Pixel}     & 48.74\% & 59.33\% & \textbf{59.56\%} & 51.67\% & 50.85\% & 50.01\% & 50.37\% & 53.60\%    \\
\multicolumn{1}{l|}{Transfer}  & 56.92\% & \textbf{63.99\%} & 63.87\% & 48.76\% & 55.92\% & 57.03\% & 58.07\% & 55.01\%  \\  
\hline 
\end{tabular}
\caption{\textit{Prediction accuracy of different methods on original samples and adversarial samples (trained with 100 FMNIST samples).}}
\label{table:accuracy_FMNIST_100}\vspace{-2em}
\end{table}

In the FMNIST experiments, we observe very similar results with the MNIST experiments. When the networks are trained with only 100 samples, almost all the regularization methods can significantly improve the prediction accuracy on the original dataset compared to the baseline method, and the graph-based methods achieve the highest accuracy. The ALP method is still the most robust with respect to the white box attacks and the gradient based methods are the most stable with respect to the black box attacks. We also notice that the GTV and GTik methods are very unstable under the PGD attack.

\begin{table}
\small
\centering
\begin{tabular}{lllllllll}
\hline
\multicolumn{1}{l|}{regularization} & baseline     & TV               & Tik     & GTV     & GTik  & L2   & AT      & ALP              \\ \hline
\multicolumn{1}{l|}{no attack}  & 88.81\%    & 88.71\%    & 88.22 \%   & 87.48\%    & \textbf{90.57\%}   & 88.62\%   & 89.04\%    & 87.03\% \\ 
\multicolumn{1}{l|}{PGD}    & 11.60\% & 40.12\% & 36.40\% & 0.48\%  & 0.12\%  & 6.54\%  & 23.60\% & \textbf{72.90\%}\\
\multicolumn{1}{l|}{One Pixel}   & 55.36\% & 71.54\% & 70.98\% & 27.78\% & 33.36\% & 53.86\% & 61.10\% & \textbf{77.20\%}   \\
\multicolumn{1}{l|}{Transfer}  & 65.38\% & \textbf{78.04\%} & 76.38\% & 21.14\% & 33.08\% & 39.17\% & 69.70\% & 24.30\%  \\  
\hline \\
\end{tabular}
\caption{\textit{Prediction accuracy of different methods on original samples and adversarial samples (trained with 10,000 FMNIST samples).}}
\label{table:accuracy_FMNIST_10000}\vspace{-2em}
\end{table}

When we train the networks using 10,000 samples, all the regularization methods achieve similar prediction accuracy on the original test set. Only the gradient-based methods can improve the stability of neural networks with respect to all attack methods. The ALP method is the most robust to the PGD attack and one pixel attack in this experiments, but its robustness against the transfer attack is much lower than the baseline. The graph-based methods are unstable under various adversarial attacks.

\section{Conclusion and future works}\label{sec6}
In this work, We use neural networks to solve some variational problems arising from machine learning applications. We give a general framework for the generalization error analysis of neural network approximations. We achieve generalization the error estimate based on two reasonable assumptions on the approximation error and the quadrature error. 

We evaluate the gradient-based regularization methods for neural networks along with their graphical extensions in the experiments. When testing the networks on the original test set, most regularization methods can outperform the baseline if the number of training samples is small and obtain similar results with the baseline method if the number of training samples is large. When the test samples are attacked by some adversarial algorithms, TV and Tikhonov are the only two methods that can significantly improve the adversarial accuracy against every attack method. 

The limitation of the gradient-based methods is the computational cost induced by computing the derivatives, which makes it difficult to apply to very large-scale problems. Besides, the approximation error of neural networks with complex structures is not easy to derive.

In the future works, we will be interested in the generalization error analysis for different types of variational models or PDEs. We may also try to apply the proposed framework to neural networks with complex structures. 

\section{Declaration}
\subsection{Funding}
This work is partially supported by the National Science Foundation of China and Hong Kong RGC Joint Research Scheme (NSFC/RGC 11961160718), and the fund of the Guangdong Provincial Key Laboratory of Computational Science and Material Design (No. 2019B030301001). The work of J. Yang is supported by the National Science Foundation of China (NSFC-11871264) and the Guangdong Basic and Applied Basic Research Foundation (2018A0303130123).
\subsection{Conflicts of interest}
Not applicable
\subsection{Availability of data and material}
Not applicable
\subsection{Code availability}
Not applicable

\appendix
\section{Proof of Lemma \ref{theorem:H1ReLU}}
\begin{proof}
For any $\eta>0$, there exist a complex measure $F(d\omega)=e^{i\zeta(\omega)}|\tilde{f}(\omega)|d\omega$ such that
\begin{align*}
    \bar{f}(x)=f(x)-f(0)&=Re\int_{\mathbb{R}^d}(e^{i\omega\cdot x}-1)F(d\omega)\\
    &=\int_{\mathbb{R}^d}(\cos(\omega\cdot x+\zeta(\omega))-\cos(\zeta(\omega)))|\tilde{f}(\omega)|d\omega\\
    &=\int_{\mathbb{R}^d}\frac{C_f}{(1+|\omega|_1)^2}\left(\cos(\omega\cdot x+\zeta(\omega))-\cos(\zeta(\omega)\right)\hat{F}(d\omega)\\
    &=\int_{\mathbb{R}^d} g(x,\omega)\hat{F}(d\omega),
\end{align*}
where $C_f=\int(1+|\omega|_1)^2|\tilde{f}(\omega)|d\omega\leq\Vert f\Vert_{\B(\Omega)}+\eta$; $\hat{F}(d\omega)=\frac{(1+|\omega|_1)^2}{C_f}|\tilde{f}(\omega)|d\omega$ is a probability distribution; and
\begin{equation*}
    g(x,\omega)=\frac{C_f}{(1+|\omega|_1)^2}\left(\cos(\omega\cdot x+\zeta(\omega))-\cos(\zeta(\omega)\right).
\end{equation*}
Define a set of function $G_{\cos}$ as:
\begin{equation}
    G_{\cos}:=\left\{\frac{\gamma}{(1+|\omega|_1)^2}\left(\cos(\omega\cdot x+b)-\cos(b)\right)\bigg| |\gamma|\leq C_f,b\in\mathbb{R}\right\}. \label{eq:Gcos}
\end{equation}
Consequently, $\bar{f}$ is in the $H^1$ closure of the convex hull of $G_{\cos}$. To see this, we suppose $\{\omega_k\}_{k=1}^m$ be a set of i.i.d. random variables generated from the distribution $\hat{F}(d\omega)$. Notice that for each $g(x,\omega)\sim\hat{F}$ and $x\in\Omega$, we have:
\begin{equation*}
|g(x,\omega)|\leq\frac{C_f}{(1+|\omega|_1)^2}|\omega|_1\leq C_f,    
\end{equation*}
and
\begin{equation*}
    |\nabla g(x,\omega)|_2\leq|\nabla g(x,\omega)|_1=\frac{C_f}{(1+|\omega|_1)^2}|\omega|_1|\sin(\omega\cdot x)|\leq C_f,
\end{equation*}
for any $x\in\Omega$.
Then, by the Fubini's theorem,
\begin{align*}
    & E_{\omega_k\sim\hat{F}}\left[\left\Vert \bar{f}-\frac{1}{m}\sum_{k=1}^m g(x,\omega_k)\right\Vert_{H^1(\Omega)}^2\right]\\
    =&\int_{\Omega}E_{\omega_k\sim\hat{F}}\left[\left|f(x)-\frac{1}{m}\sum_{k=1}^m g(x,\omega_k)\right|^2+\left|\nabla f(x)-\frac{1}{m}\sum_{k=1}^m \nabla g(x,\omega_k)\right|^2_2\right] dx \\
    \leq&\frac{1}{m}\int_{\Omega}E_{\omega\sim\hat{F}}\left[\left|g(x,\omega)\right|^2+\left|\nabla g(x,\omega)\right|_2^2\right] dx\\
    \leq& \frac{2C_f^2}{m}.
\end{align*}
The expected $H^1$ norm error would converge to 0 as $m$ goes to infinity, which implies that there exist a sequence of convex combination of functions in $G_{\cos}$ converges to $\bar{f}$.

We next define a one dimensional function $g(z)=\gamma(\cos(|\omega|_1z+b)-\cos(b))/(1+|\omega|_1)^2$ where $\gamma\leq C_f$ and $z\in[-1,1]$. Notice that $g(x,\omega)$ is just the composition of $g(z)$ and $z=\omega\cdot x/|\omega|_1$. Let $\tilde{g}(z)=g(z)+\gamma\sin(b)|\omega|_1z/(1+|\omega|_1)^2$. We can examine that $\Vert \tilde{g}^{(s)}\Vert_{L^{\infty}([-1,1])}\leq 2C_f$ for $s=0,1,2$ and $\tilde{g}(0)=\tilde{g}'(0)=0$.
\begin{lemma}[Lemma 4.5 in \cite{lu2021priori}]\label{lemma:linearinterpolation}
Given $g\in C^2([-1,1])$ with $\Vert g^{(s)}\Vert_{L^{\infty}([-1,1])}\leq B$ and $g'(0)=0$ for $s=0,1,2$. Define $\{z_k\}_{k=0}^{2m}=\{-1+k/m\}_{i=0}^{2m}$ be a uniform partition of $[-1,1]$. Then, there exist a ReLU network $g_m$ of the form:
\begin{equation*}
    g_m(z)=c+\frac{1}{2m}\sum_{k=1}^{2m}a_i\relu(\epsilon_kz+b_k),
\end{equation*}
where $c=g(0)$, $|a_k|\leq 4B$, $\epsilon_k\in\{-1,1\}$ and $|b_k|\leq 1$, such that
\begin{equation*}
    \Vert g-g_m\Vert_{W^{1,\infty}([-1,1])}\leq\frac{2B}{m}.
\end{equation*}
\end{lemma}
Consequently, for any $g\in G_{\cos}$, there exist a convex combination of $\relu$ functions such that
\begin{align*}
    &\left\Vert \tilde{g}-\frac{1}{2m}\sum_{k=1}^{2m}a_i\relu(\epsilon_kz+b_k)\right\Vert_{W^{1,\infty}([-1,1])}\\
    =&\left\Vert g+\frac{\gamma\sin(b)|\omega|_1z}{(1+|\omega|_1)^2}-\frac{1}{2m}\sum_{k=1}^{2m}a_i\relu(\epsilon_kz+b_k)\right\Vert_{W^{1,\infty}([-1,1])}\\
    =&\left\Vert g+\frac{1}{2}(2a_0\relu(z+1))-\frac{1}{4m}\sum_{k=1}^{2m}2a_i\relu(\epsilon_kz+b_k)-a_0\right\Vert_{W^{1,\infty}([-1,1])}\\
    \leq&\frac{4C_f}{m},
\end{align*}
where $a_0=\frac{\gamma\sin(b)|\omega|_1}{(1+|\omega|_1)^2}$, $|a_0|\leq C_f$, $|a_k|\leq 8C_f$, $\epsilon_k\in\{-1,1\}$, and $|b_k|\leq 1$ for $k=1,\dots,2m$. Therefore, $g$ lies in the $H^1$-closure of the convex hull of the following set:
\begin{equation}
    G_{\relu}:=\left\{c+a\relu(\omega\cdot x+b)\big||a|\leq 16C_f,|\omega|_1=1,|b|\leq 1,|c|\leq C_f\right\}, \label{eq:GReLU}
\end{equation}
and so does $\bar{f}$.
For any $g\in G_{\relu}$, we have $\Vert g\Vert_{H^1(\Omega)}^2\leq ((32C_f+C_f)^2+(16C_f)^2)=1345C_f^2$.
Therefore, by the Mauray Lemma (Lemma 1 in \cite{barron1993universal}), there exist a two-layer ReLU network 
$$f_m=c+\frac{1}{m}\sum_{k=1}^ma_i\relu(\omega_k\cdot x+b_k),$$ 
where $|a_k|\leq 16C_f$, $|\omega_k|_1=1$, $|b_k|\leq 1$ and $|c|\leq 2C_f$, such that
\begin{equation*}
    \Vert f-f_m\Vert_{H^1(\Omega)}\leq \frac{\sqrt{1345}C_f}{\sqrt{m}}\leq \frac{\sqrt{1345}(\Vert f\Vert_{\B(\Omega)}+\eta)}{\sqrt{m}}.
\end{equation*}
Notice that $f(0)=\int_{\mathbb{R}^d}F(d\omega)\leq C_f$ has been added to the constant term $c$.
$\hfill\square$
\end{proof}

\section{Proof of Lemma \ref{theorem:H1SP}}
\begin{proof} 
Define a set of SP functions as
\begin{equation}
    G_{\SP_\tau}:=\{c+a\SP_\tau(\omega\cdot x+b)\big||a|\leq 16C_f,|\omega|_1=1,|b|\leq 1,|c|\leq C_f\},\label{eq:GSP}
\end{equation}
where $C_f$ is defined in the proof of Lemma \ref{theorem:H1ReLU}.
For each $g\in G_{\cos}$, there exist a ReLU network $g_m\in G_{\relu}$ such that $\Vert g-g_m\Vert_{W^{1,\infty}([-1,1])}$. If we replace the activation function of $g_m$ by $\SP_{\tau}$, we get a softplus network $g_{\tau,m}\in G_{\SP_{\tau}}$. Then, from the proof of lemma 4.7 in \cite{lu2021priori},
\begin{align*}
    \Vert g-g_{\tau,m}\Vert_{W^{1,\infty}([-1,1])}\leq&\Vert g-g_m\Vert_{W^{1,\infty}([-1,1])}+\Vert g_{\tau,m}-g_m\Vert_{W^{1,\infty}([-1,1])}\\
    \leq& \frac{4C_f}{m}+(1+\frac{1}{\tau})(\frac{8C_f}{m}+4C_fe^{-\tau/m})\\
    \leq& 12C_f\delta_\tau,
\end{align*}
where $\delta_\tau=\frac{1}{\tau}(1+\frac{1}{\tau})(\log(\frac{\tau}{3})+1)$. In Lemma \ref{theorem:H1ReLU}, we have proven that $\bar{f}(x)=f(x)-f(0)$ lies in the closure of the convex hull of $G_{\cos}$. Thus, there exist a function $f_\tau$ in the closure of the convex hull of $G_{\SP_\tau}$ such that
\begin{equation*}
    \Vert \bar{f}-f_\tau\Vert_{W^{1,\infty}([-1,1])}\leq 12C_f\delta_\tau.
\end{equation*}
For any $g_\tau\in G_{\SP_\tau}$, we can verify that
\begin{align*}
    \Vert g_\tau\Vert_{H^1(\Omega)}^2\leq& (16C_f)^2\left(\Vert\SP_\tau\Vert_{L^\infty([-2,2])}^2+\Vert\SP_\tau’\Vert_{L^\infty([-2,2])}^2\right)+(C_f)^2\\
    \leq& (16C_f)^2\left(|SP_\tau(2)|^2+\left|\frac{1}{1+e^{-2\tau}}\right|^2\right)+(C_f)^2\\
    \leq & (16C_f)^2\left(\left|\frac{1}{\tau}+2\right|^2+1\right)+(C_f)^2\\
    \leq & (16C_f)^2\left(\frac{1}{\tau}+3\right)^2+(C_f)^2.
\end{align*}
By Mauray's lemma, there exist a two-layer SP networks $f_m$, which is the convex combination of $m$ functions in $G_{\SP_\tau}$, such that
\begin{align*}
    \Vert \bar{f}-f_m\Vert_{H^1(\Omega)}=& \Vert f-(f(0)+f_m)\Vert_{H^1(\Omega)}\\
    \leq& \Vert \bar{f}-f_\tau\Vert_{H^1(\Omega)}+\Vert f_\tau-f_m\Vert_{H^1(\Omega)}\\
    \leq & 12C_f\delta_\tau+\frac{C_f(16(3+1/\tau)+1)}{\sqrt{m}}.
\end{align*}
Using the fact that $\tau=\sqrt{m}\geq 1$, we get
\begin{align*}
    &\Vert f-(f(0)+f_m)\Vert_{H^1(\Omega)}\\
    \leq&\frac{C_f}{\sqrt{m}}(24\log(m)+65)\\
    \leq&\frac{(\Vert f\Vert_{\B}+\eta)}{\sqrt{m}}\left(24\log(m)+65\right).
\end{align*}
$\hfill\square$
\end{proof}

\section{Proof of Lemma \ref{lemma:quadratureerror}}
\begin{proof}
Let $\mathcal{F}$ be a class of functions defined on $\Omega$, $(\sigma_i)_{i=1}^N$ be a set of i.i.d variables drawn from the distribution $P(\sigma_i=1)=P(\sigma_i=-1)=1/2$, and $(x_i)_{i}^N$ is a set of i.i.d. variables sampled from $\Omega$. Then, the the Rademacher complexity of $\mathcal{F}$ is defined as
\begin{equation*}
    \mathcal{R}_N(\mathcal{F})=E_{x_1,\dots,x_N}E_\sigma\left[\sup_{f\in\mathcal{F}}\left|\frac{1}{N}\sum_{i=1}^N\sigma_if(x_i)\right|\right].
\end{equation*}
\begin{lemma}[Section 5 in \cite{lu2021priori}]\label{lemma:appendix1}
Define the function class $\mathcal{F}^1(B)$, $\mathcal{F}^2(B)$ and $\mathcal{F}^3(B)$ as
\begin{align*}
    &\mathcal{F}^1(B):=\left\{\phi^2-2\phi y\bigg|\phi\in\mathcal{F}^m_{\SP_{\tau}(B)}\right\}\\
    &\mathcal{F}^2(B):=\left\{|\nabla\phi|\bigg|\phi\in\mathcal{F}^m_{\SP_{\tau}(B)}\right\}\\
    &\mathcal{F}^3(B):=\left\{|\nabla\phi|^2\bigg|\phi\in\mathcal{F}^m_{\SP_{\tau}(B)}\right\},
\end{align*}
where $y$ is assumed to be bounded by $Y$. Then we have
\begin{align*}
    &\mathcal{R}_N(\mathcal{F}^1(B))\leq Z_1\frac{\sqrt{m}}{\sqrt{N}}\\
    &\mathcal{R}_N(\mathcal{F}^2(B))\leq Z_2\frac{\sqrt{m\log(m)}}{\sqrt{N}}\\
    &\mathcal{R}_N(\mathcal{F}^3(B))\leq Z_3\frac{\sqrt{m\log(m)}}{\sqrt{N}},
\end{align*}
where $Z_1$ is a constant depends on $B$, $d$ and $Y$; $Z_2$ and $Z_3$ are constants depends on $B$ and $d$. 
\end{lemma}
\begin{lemma}[Theorem 4.10 in \cite{wainwright2019high}]\label{lemma:appendix2}
Let $\mathcal{F}$ be a class of functions uniformly bounded by $M$ and $(x_1,\dots,x_N)$ be i.i.d sampled from $\Omega$, then for any $0<\epsilon<1$,
\begin{equation*}
    \sup_{f\in\mathcal{F}}\left|\frac{1}{N}\sum_{i=1}^Nf(x_i)-E_x(f(x))\right|\leq 2\mathcal{R}_N(\mathcal{F})+M\frac{\sqrt{-2\log(\epsilon)}}{\sqrt{N}},
\end{equation*}
with probability at least $1-\epsilon$. 
\end{lemma}
Combining Lemma \ref{lemma:appendix1} and \ref{lemma:appendix2}, we have probability of
\begin{equation*}
    \sup_{\phi_m\in\mathcal{F}^m_{\SP_{\tau}}(B)}\left|L^{TV}_N(\phi_m)-L^{TV}(\phi_m)\right|\leq 2\frac{\sqrt{m}}{\sqrt{N}}(Z_1+Z_2\sqrt{\log(m)})+M_1\frac{\sqrt{-2\log(\epsilon)}}{\sqrt{N}}
\end{equation*}
and
\begin{equation*}
     \sup_{\phi_m\in\mathcal{F}^m_{\SP_{\tau}}(B)}\left|L^{Tik}_N(\phi_m)-L^{Tik}(\phi_m)\right|\leq 2\frac{\sqrt{m}}{\sqrt{N}}(Z_1+Z_3\sqrt{\log(m)})+M_2\frac{\sqrt{-2\log(\epsilon)}}{\sqrt{N}},
\end{equation*}
are at least $1-\epsilon$ respectively, where $M_1=\sup_{f_1\in\mathcal{F}^1(B),f_2\in\mathcal{F}^2(B)}\Vert f_1+f_2\Vert_{L^{\infty}(\Omega)}$ and $M_2=\sup_{f_1\in\mathcal{F}^1(B),f_2\in\mathcal{F}^3(B)}\Vert f_1+f_2\Vert_{L^{\infty}(\Omega)}$ are constants depends on $B$ and $Y$.
$\hfill\square$
\end{proof}
\section{CNN Architecture}
\textbf{The CNN model used in the MNIST expreiments:}
\begin{lstlisting}[basicstyle=\small]
Model: "CNN_MNIST"
_______________________________________________________________
Layer (type)                 Output Shape              Param #
===============================================================
conv_1 (Conv2D)              (None, 28, 28, 16)        160
_______________________________________________________________
pooling_1 (MaxPooling2D)     (None, 14, 14, 16)        0
_______________________________________________________________
relu_1 (Activation)          (None, 14, 14, 16)        0
_______________________________________________________________
conv_2 (Conv2D)              (None, 14, 14, 64)        9280
_______________________________________________________________
relu_2 (Activation)          (None, 14, 14, 64)        0
_______________________________________________________________
pooling_2 (MaxPooling2)      (None, 7, 7, 64)          0
_______________________________________________________________
flatten (Flatten)            (None, 3136)              0
_______________________________________________________________
dense (Dense)                (None, 1000)              3137000
_______________________________________________________________
activation_2 (Activation)    (None, 1000)              0
_______________________________________________________________
dense_1 (Dense)              (None, 10)                10010
===============================================================
Total params: 3,156,450
Trainable params: 3,156,450
Non-trainable params: 0
_______________________________________________________________
\end{lstlisting}
\textbf{The CNN model used in the Fashion MNIST expreiments:}
\begin{lstlisting}[basicstyle=\small]
Model: "CNN_FMNIST"
_______________________________________________________________
Layer (type)                 Output Shape              Param #
===============================================================
conv_1 (Conv2D)              (None, 28, 28, 16)        160
_______________________________________________________________
relu_1 (Activation)          (None, 28, 28, 16)        0
_______________________________________________________________
conv_2 (Conv2D)              (None, 28, 28, 16)        2320
_______________________________________________________________
pooling_1 (MaxPooling2D)     (None, 14, 14, 16)        0
_______________________________________________________________
relu_2 (Activation)          (None, 14, 14, 16)        0
_______________________________________________________________
conv_3 (Conv2D)              (None, 14, 14, 64)        9280
_______________________________________________________________
relu_3 (Activation)          (None, 14, 14, 64)        0
_______________________________________________________________
conv_4 (Conv2D)              (None, 14, 14, 64)        36928
_______________________________________________________________
pooling_2 (MaxPooling2)      (None, 7, 7, 64)          0
_______________________________________________________________
relu_4 (Activation)          (None, 7, 7, 64)          0
_______________________________________________________________
flatten (Flatten)            (None, 3136)              0
_______________________________________________________________
dense_1 (Dense)              (None, 1000)              3137000
_______________________________________________________________
relu_5 (Activation)          (None, 1000)              0
_______________________________________________________________
dense_2 (Dense)              (None, 10)                10010
===============================================================
Total params: 3,195,698
Trainable params: 3,195,698
Non-trainable params: 0
_______________________________________________________________

\end{lstlisting}

\bibliographystyle{unsrt}
\bibliography{references}

\begin{thebibliography}{10}

\bibitem{goodfellow2016deep}
Ian Goodfellow, Yoshua Bengio, Aaron Courville, and Yoshua Bengio.
\newblock {\em Deep learning}.
\newblock MIT Press Cambridge, 2016.

\bibitem{higham2019deep}
Catherine~F Higham and Desmond~J Higham.
\newblock Deep learning: An introduction for applied mathematicians.
\newblock {\em SIAM Review}, 61(4):860--891, 2019.

\bibitem{sirignano2018dgm}
Justin Sirignano and Konstantinos Spiliopoulos.
\newblock Dgm: A deep learning algorithm for solving partial differential
  equations.
\newblock {\em Journal of Computational Physics}, 375:1339--1364, 2018.

\bibitem{weinan2018deep}
E~Weinan and Bing Yu.
\newblock The deep \uppercase{R}itz method: A deep learning-based numerical
  algorithm for solving variational problems.
\newblock {\em Communications in Mathematics and Statistics}, 6(1):1--12, 2018.

\bibitem{han2018solving}
Jiequn Han, Arnulf Jentzen, and E~Weinan.
\newblock Solving high-dimensional partial differential equations using deep
  learning.
\newblock {\em Proceedings of the National Academy of Sciences},
  115(34):8505--8510, 2018.

\bibitem{barron1993universal}
Andrew~R Barron.
\newblock Universal approximation bounds for superpositions of a sigmoidal
  function.
\newblock {\em IEEE Transactions on Information Theory}, 39(3):930--945, 1993.

\bibitem{cybenko1989approximation}
George Cybenko.
\newblock Approximation by superpositions of a sigmoidal function.
\newblock {\em Mathematics of Control, Signals and Systems}, 2(4):303--314,
  1989.

\bibitem{hornik1989multilayer}
Kurt Hornik, Maxwell Stinchcombe, Halbert White, et~al.
\newblock Multilayer feedforward networks are universal approximators.
\newblock {\em Neural Networks}, 2(5):359--366, 1989.

\bibitem{zhou2020universality}
Ding-Xuan Zhou.
\newblock Universality of deep convolutional neural networks.
\newblock {\em Applied and Computational Harmonic Analysis}, 48(2):787--794,
  2020.

\bibitem{goodfellow2015explaining}
Ian Goodfellow, Jonathon Shlens, and Christian Szegedy.
\newblock Explaining and harnessing adversarial examples.
\newblock In {\em International Conference on Learning Representations}, 2015.

\bibitem{szegedy2013intriguing}
Christian Szegedy, Wojciech Zaremba, Ilya Sutskever, Joan Bruna, Dumitru Erhan,
  Ian Goodfellow, and Rob Fergus.
\newblock Intriguing properties of neural networks.
\newblock In {\em International Conference on Learning Representations}, 2014.

\bibitem{madry2018towards}
Aleksander Madry, Aleksandar Makelov, Ludwig Schmidt, Dimitris Tsipras, and
  Adrian Vladu.
\newblock Towards deep learning models resistant to adversarial attacks.
\newblock In {\em International Conference on Learning Representations}, 2018.

\bibitem{su2019one}
Jiawei Su, Danilo~Vasconcellos Vargas, and Kouichi Sakurai.
\newblock One pixel attack for fooling deep neural networks.
\newblock {\em IEEE Transactions on Evolutionary Computation}, 23(5):828--841,
  2019.

\bibitem{kannan2018adversarial}
Harini Kannan, Alexey Kurakin, and Ian Goodfellow.
\newblock Adversarial logit pairing.
\newblock {\em arXiv preprint arXiv:1803.06373}, 2018.

\bibitem{rudin1992nonlinear}
Leonid~I Rudin, Stanley Osher, and Emad Fatemi.
\newblock Nonlinear total variation based noise removal algorithms.
\newblock {\em Physica D: Nonlinear Phenomena}, 60(1-4):259--268, 1992.

\bibitem{chan2001active}
Tony~F Chan and Luminita~A Vese.
\newblock Active contours without edges.
\newblock {\em IEEE Transactions on Image Processing}, 10(2):266--277, 2001.

\bibitem{vese2002multiphase}
Luminita~A Vese and Tony~F Chan.
\newblock A multiphase level set framework for image segmentation using the
  \uppercase{M}umford and \uppercase{S}hah model.
\newblock {\em International Journal of Computer Vision}, 50(3):271--293, 2002.

\bibitem{tikhonov2013numerical}
Andrei~Nikolaevich Tikhonov, AV~Goncharsky, VV~Stepanov, and Anatoly~G Yagola.
\newblock {\em Numerical methods for the solution of ill-posed problems},
  volume 328.
\newblock Springer Science \& Business Media, 2013.

\bibitem{oberman2020partial}
Adam~M Oberman.
\newblock Partial differential equation regularization for supervised machine
  learning.
\newblock In {\em 75 Years of Mathematics of Computation: Symposium on
  Celebrating 75 Years of Mathematics of Computation, November 1-3, 2018, the
  Institute for Computational and Experimental Research in Mathematics (ICERM),
  Providence, Rhode Island}, volume 754, page 177. American Mathematical Soc.,
  2020.

\bibitem{raissi2018hidden}
Maziar Raissi and George~Em Karniadakis.
\newblock Hidden physics models: Machine learning of nonlinear partial
  differential equations.
\newblock {\em Journal of Computational Physics}, 357:125--141, 2018.

\bibitem{raissi2019physics}
Maziar Raissi, Paris Perdikaris, and George~E Karniadakis.
\newblock Physics-informed neural networks: A deep learning framework for
  solving forward and inverse problems involving nonlinear partial differential
  equations.
\newblock {\em Journal of Computational Physics}, 378:686--707, 2019.

\bibitem{mishra2020estimates}
Siddhartha Mishra and Roberto Molinaro.
\newblock Estimates on the generalization error of physics informed neural
  networks (pinns) for approximating pdes.
\newblock {\em SAM Research Report}, 2020, 2020.

\bibitem{mishra2020estimates2}
Siddhartha Mishra and Roberto Molinaro.
\newblock Estimates on the generalization error of physics informed neural
  networks (pinns) for approximating pdes ii: A class of inverse problems.
\newblock {\em SAM Research Report}, 2020, 2020.

\bibitem{lu2021priori}
Jianfeng Lu, Yulong Lu, and Min Wang.
\newblock A priori generalization analysis of the deep ritz method for solving
  high dimensional elliptic equations.
\newblock {\em arXiv preprint arXiv:2101.01708}, 2021.

\bibitem{muller2021error}
Johannes Müller and Marius Zeinhofer.
\newblock Error estimates for the variational training of neural networks with
  boundary penalty.
\newblock {\em arXiv preprint arXiv:2103.01007}, 2021.

\bibitem{chambolle2004algorithm}
Antonin Chambolle.
\newblock An algorithm for total variation minimization and applications.
\newblock {\em Journal of Mathematical Imaging and Vision}, 20(1-2):89--97,
  2004.

\bibitem{wu2010augmented}
Chunlin Wu and Xue-Cheng Tai.
\newblock Augmented lagrangian method, dual methods, and split bregman
  iteration for \uppercase{ROF}, vectorial \uppercase{TV}, and high order
  models.
\newblock {\em SIAM Journal on Imaging Sciences}, 3(3):300--339, 2010.

\bibitem{goldstein2009split}
Tom Goldstein and Stanley Osher.
\newblock The split bregman method for l1-regularized problems.
\newblock {\em SIAM Journal on Imaging Sciences}, 2(2):323--343, 2009.

\bibitem{chan1999nonlinear}
Tony~F Chan, Gene~H Golub, and Pep Mulet.
\newblock A nonlinear primal-dual method for total variation-based image
  restoration.
\newblock {\em SIAM Journal on Scientific Computing}, 20(6):1964--1977, 1999.

\bibitem{yuan2010study}
Jing Yuan, Egil Bae, and Xue-Cheng Tai.
\newblock A study on continuous max-flow and min-cut approaches.
\newblock In {\em 2010 Ieee Computer Society Conference on Computer Vision and
  Pattern Recognition}, pages 2217--2224. IEEE, 2010.

\bibitem{drucker1992improving}
Harris Drucker and Yann Le~Cun.
\newblock Improving generalization performance using double backpropagation.
\newblock {\em IEEE Transactions on Neural Networks}, 3(6):991--997, 1992.

\bibitem{bishop1995training}
Chris~M Bishop.
\newblock Training with noise is equivalent to \uppercase{T}ikhonov
  regularization.
\newblock {\em Neural Computation}, 7(1):108--116, 1995.

\bibitem{wang2020adversarial}
Bao Wang, Alex Lin, Penghang Yin, Wei Zhu, Andrea~L Bertozzi, and Stanley~J
  Osher.
\newblock Adversarial defense via the data-dependent activation, total
  variation minimization, and adversarial training.
\newblock {\em Inverse Problems and Imaging}, 15(1):129, 2020.

\bibitem{papernot2017practical}
Nicolas Papernot, Patrick McDaniel, Ian Goodfellow, Somesh Jha, Z~Berkay Celik,
  and Ananthram Swami.
\newblock Practical black-box attacks against machine learning.
\newblock In {\em Proceedings of the 2017 ACM on Asia conference on computer
  and communications security}, pages 506--519, 2017.

\bibitem{Kingma2015Adam}
Diederik~P. Kingma and Jimmy Ba.
\newblock Adam: {A} method for stochastic optimization.
\newblock In Yoshua Bengio and Yann LeCun, editors, {\em 3rd International
  Conference on Learning Representations, {ICLR} 2015, San Diego, CA, USA, May
  7-9, 2015, Conference Track Proceedings}, 2015.

\bibitem{loshchilov2018decoupled}
Ilya Loshchilov and Frank Hutter.
\newblock Decoupled weight decay regularization.
\newblock In {\em International Conference on Learning Representations}, 2018.

\bibitem{wainwright2019high}
Martin~J Wainwright.
\newblock {\em High-dimensional statistics: A non-asymptotic viewpoint}.
\newblock Cambridge University Press, 2019.

\end{thebibliography}

\end{document}